\documentclass{article}
\usepackage{ijcai19}

\usepackage{times}
\usepackage{soul}
\usepackage{url}
\usepackage[hidelinks]{hyperref}
\usepackage{enumerate}
\usepackage[utf8]{inputenc}
\usepackage[small]{caption}
\usepackage{graphicx}
\usepackage{amsmath}
\usepackage{booktabs}
\usepackage{algorithm}
\usepackage{algorithmic}
\usepackage{subcaption}
\usepackage{times}
\usepackage{epsfig}
\usepackage{graphicx}
\usepackage{amsmath}
\usepackage{amssymb}
\usepackage{times}
\usepackage{epsfig}
\usepackage{caption}
\usepackage{subcaption}
\usepackage{graphicx}
\usepackage{amsmath}
\usepackage{amssymb}
\usepackage{adjustbox}
\usepackage{enumitem}
\usepackage{amsthm}
\usepackage{amsmath}

\newtheorem{theorem}{Theorem}

\graphicspath{{figures/}{image_samples/Best/}{image_samples/Intro/}{images/}}
\newcommand{\tablepiclarge}[1]{\includegraphics[width=.25\linewidth,height=.25\linewidth,,trim={1.2cm .8cm .4cm .4cm},clip]{#1}}

\title{Regression with Uncertainty Quantification in Large Scale Complex Data}

\author{
    Anonymous
}

\author{
Nicholas Wilkins$^1$
\and
Michael Johnson$^1$
\and
Ifeoma Nwogu$^1$\and
\affiliations
$^1$Rochester Institute of Technology\\
\emails
\{npw3202, mxj5897, ionvcs\}@rit.edu,
}

\begin{document}

\maketitle
\begin{abstract}
While several methods for predicting uncertainty on deep networks have been recently proposed, they do not readily translate to large and complex datasets. In  this  paper we  utilize a simplified form of the Mixture Density Networks (MDNs) to produce a  one-shot approach to  quantify uncertainty in regression problems. We show that our uncertainty bounds are on-par or  better  than  other  reported existing  methods.  When applied to standard regression benchmark datasets, we show an improvement in predictive log-likelihood and root-mean-square-error when compared to existing state-of-the-art methods. We also demonstrate this method's efficacy on stochastic, highly volatile time-series data where stock prices are predicted for the next time interval. The resulting uncertainty graph summarizes significant anomalies in the stock price chart. Furthermore, we apply this method to the task of age estimation from the challenging IMDb-Wiki dataset of half a million face images. We successfully predict the uncertainties associated with the prediction and empirically analyze the underlying causes of the uncertainties. This uncertainty quantification can be used to pre-process low quality datasets and further enable learning.
\end{abstract}

\section{Introduction}
When utilizing regression under deep learning, one typically attempts to learn an optimal mapping (under some loss function) from a feature space (notated here as $X$) to some target space (notated here as $Y$): $\hat{f}: X \rightarrow Y$. We wish to learn that function such that some loss function $\mathcal{L}$ is minimized. Due to the aforementioned construction, typically when utilizing a regressor, a single value is predicted. Frequently, however, the feature space does not adequately capture enough information to perfectly predict $Y$. Additionally, there is frequently stochasticity present within the system (i.e. aleatoric uncertainty) which prevents any system from perfectly predicting onto $Y$. The main focus of this work is to construct such a regressor which efficiently regresses onto a normal distribution on the target space. The choice of a Gaussian as the target distribution is in standing with traditional statistics methods where when the measurement errors occurring in regression problems are assumed to follow a normal distribution. 

Deep Neural Networks have transformed the field of machine learning by allowing advanced concepts to be learned from large-scale input data. Furthermore, these techniques have allowed for recent breakthroughs in pattern recognition that have applications in many fields such as chemistry, biology, physics, manufacturing, and medical sciences. 
Although promising and highly useful, many deep learning techniques only provide a point estimate and seldom provide a means to understand inherent uncertainty in the input data. Due to this, they are frequently incapable of understanding their own limitations\footnote{Although in classification one can determine how far training samples are from the decision boundary, this confidence is often significantly different from understanding the inherent limitations of the learning system. This can potentially have disastrous impacts in many important real life scenarios}. 

For many areas of scientific study, especially in areas of critical importance such as in medical image analysis for patient diagnosis, this lack of uncertainty quantification is highly problematic. The inability to understand and quantify the model's confidence in its predicted values is a high source of potential risk and liability \cite{Medical_Imaging}. For example, when faced with a difficult diagnosis, the ability for a deep learning system to report large uncertainties would allow for human operators to intervene and review those specific cases. \emph{If deep learning is to be widely used for critical applications in practical settings, such as making medical diagnosis from input data, a key requirement would be the ability to provide \underline{statistically meaningful} uncertainty measurements with their predictions}.

In his 1994 Ph.D. thesis, Bishop \cite{Bishop94} introduced Mixture Density Networks (MDNs); where a neural network is used to predict a probability distribution over the target value $Y$, rather than a single point estimate. The MDNs train with a fixed number of mixtures of Gaussian components over the course of the training scheme using the Negative Log Likelihood (NLL) as the loss function to the network. This proposed training scheme has the potential to address many of the issues highlighted above, but due to limitations in computing power in the 1990's, MDNs did not gain as wide popularity.

In this paper we present a elegant and simplified approach to quantify uncertainty in large-scale regression problems. We propose a one-shot approach requiring no significant overhead. Additionally, we demonstrate the uncertainty bounds produced by this system are on-par or better than currently existing methods reported in the literature. Finally, we illustrate how this can be utilized for cleaning datasets and removing erroneous data autonomously. 


\section{Prior Work}\label{sec:priorwork}
As uncertainty measurement is immensely useful, prior research has been conducted into quantifying uncertainty. However, many proposed methods require substantial overhead. Additionally, although many past works have made the distinction, our calculated uncertainty is capable of incorporating both aleatory and epistemic uncertainties\cite{uncert}. \footnote{Aleatory Uncertainty is irreducible uncertainty due to probabilistic variability while epistemic uncertainty is reducible and stems from uncertainty in the system model.}

\paragraph{\sc{Separate Regressor}}
One popular method for quantifying uncertainty is to regress directly on the uncertainty\cite{sepred}. Typically, two regressors are utilized: a \textit{value} regressor and an \textit{uncertainty} regressor. These work separately to predict their respective values. This method requires the uncertainty regressor to learn the specifications of the value regressor. Additionally, the training schedule must be carefully designed to ensure that both regressors learn in tandem. Furthermore, it is much easier (due to the complexity of simultaneously optimizing two systems) for this system to get stuck in a local optimum. 

In our proposed approach we utilize a single network, thereby allowing for various components of the value regressor to interact with components of the uncertainty regressor (and vice versa). This reduces the computational overhead introduced by having two separate regressors. Furthermore, as we are only training a single network, our training schedule is significantly less complex. 

\paragraph{\sc{Deep Gaussian Process}}
Deep Gaussian processes are a class of models utilized for regression that combine Gaussian processes (GPs) with deep architectures. These were initially introduced by \cite{Deep_gaussian}. Deep GP is a composition of GP's where each layer $l$ consists of $D^{(l)}$ GP units that connect it to the next layer. Imagine a neural network with one or more hidden nodes and each edge connecting any two nodes in the network is a GP.  Exact inference on deep GPs is intractable, and although several variational approximation methods have been proposed, they are difficult to implement and do not extend readily to arbitrary kernels. They can be used to perform a regression with uncertainty bounds through a technique known as Kriging, but struggle with high dimensional or large-scale data. 

Because our proposed approach does not bear the significant overhead burden when training, it is capable of regressing on high-dimensional, large scale image data and provides uncertainty measures on the predictions made.

\paragraph{\sc{MC Dropout}}
MC Dropout\cite{MCDrop}, a method aimed at replicating the behavior of a Deep Gaussian process, can also be used for uncertainty quantification. However, utilizing this method requires ensembling on the network, leading to multiple evaluations of the network, which results in additional computational expense. 

As we regress directly on the distribution parameters, we do not need to sample from our network thus making the network substantially faster (as we only need to make one forward pass through our network).

\paragraph{\sc{Ensemble Methods}}
Recently there has been research into utilizing Deep Ensembles\cite{Deep_ensemble} (an ensemble of deep learners) to create multiple hypotheses. From these hypotheses, an uncertainty can be inferred. While extremely promising, utilizing this method requires one to train multiple deep learners and evaluate multiple deep networks to generate uncertainty resulting in a fairly computationally expensive process. We avoid these issues by utilizing a single regressor. 

Since we are only utilizing a single regressor, we only need to train, evaluate, and store one regressor.

\paragraph{\sc{Bayesian Regression}}
Drawing inspiration from statistics and probability theory, Bayesian regression assigns each parameter a prior probability distribution. Bayesian learning, via the Bayesian update rule, is utilized to update the probability distributions to best fit the data. One can extend Bayesian Regression to neural networks\cite{darkKnowledge}\cite{BayesBackprop}\cite{ProbBackprop} utilizing a similar methodology to produce uncertainty. As before, this requires storing and optimizing a distribution for each parameter, and thus is computationally expensive. Additionally, to determine a hypotheses and uncertainty, one must sample the network utilizing techniques such as variations inference\cite{VI} or MCMC\cite{MCMC}. 

We are performing ordinary regression on the distribution parameters; each of our weights and biases take on a single value. Thus, we do not need to sample from our network.  This enables us to use our method on large scale imaging datasets.

%

\section{Methodology}
\subsection{Framing the Problem}
\label{subsec:framing_the_problem}
Suppose we have samples $\left\{(x_i, y_i)\right\}_{i = 1}^N \sim \mathcal{D}(X,Y)$ where $\mathcal{D}(X,Y)$ is a joint probability distribution of $X$ and $Y$. For this paper, we assume that 
\begin{align}
\mathcal{D}(X,Y)_{X = x_0} = \mathcal{N}(\mu_{x_0}, \Sigma_{x_0}).
\end{align} 
That is to say that each cross section of the joint probability distribution function (PDF) degenerates into a normal distribution. Furthermore, we assume that each output dimension is conditionally independent from each other. Thus, for all $x_0$, $\Sigma_{x_0}$ is a diagonal matrix. 

We wish to learn a mapping from $X$ to the parameters of a Gaussian: 
\begin{align}
\phi: X \rightarrow \mathbb{R} \times (\mathbb{R}^+-\{0\})
\end{align}
where $\phi(x_0) = (\mu_{x_0},\sigma_{x_0}) $
so that 
\begin{align}
\mathcal{D}(X, Y)_{X = x_0} \stackrel{d}{=} \mathcal{N}(\phi(x_0))
\end{align}
(or alternatively the KL Divergence is minimized). Utilizing this mapping, one can determine the uncertainty (both epistemic and aleatory) of our model. We demonstrate the capability of capturing both of these uncertainties in the experiments section. 

Thus, as our target distributions are Gaussian, by producing a distribution on target variable, one can produce confidence intervals on the target variable. If one can determine such a mapping, one can achieve the uncertainty quantification described above.

\subsection{Approach}
To learn the mapping described in Section \ref{subsec:framing_the_problem}, we train a regressor to output the parameters of our target distribution with the following log-likelihood loss: 
\begin{align}
\mathcal{L} = -\iint\limits_S \log(\rho_{X}(Y)) p(X,Y) dS
\end{align}
where $p(X,Y)$ is the true joint distribution on $X,Y$ and $\rho_x$ is a Gaussian induced with parameters from the regressor. In the appendix we demonstrate that an optimal learning scheme (with appropriate assumptions) will converge to the true distribution parameters assuming the target distribution is a Gaussian. Thus, a scheme which is optimal under this loss will also have a minimal mean squared error (or any other loss) to the target data points.

This loss under finite data degenerates into NLL loss:
\begin{align}
\mathcal{L} &= -\sum_{x \in X}\sum_{y \in Y} \log(\rho_{x}(y)) f(x,y)\\
           &= -\sum_{i = 1}^N \log(\rho_{x_i}(y_i))
\end{align}
where $f(x,y)$ is the frequency $(x,y)$ occurs in the dataset. As our target distribution is a Gaussian\footnote{for this simplification, please note that $\log$ is base $e$},
\begin{align}
\mathcal{L} &= -\sum_{i = 1}^N \log\left(\frac{1}{\sqrt{2\pi\sigma_{x_i}^2}}e^{-\frac{(y_i-\mu_{x_i})^2}{2\sigma_{x_i}^2}}\right)\\
&=\frac{1}{2}\sum_{i = 1}^N \left(\left(\frac{y_i-\mu_{x_i}}{\sigma_{x_i}}\right)^2+ \log(2\pi)+2\log(\sigma_{x_i})\right)
\end{align}
which will reach its minimum where
\begin{align}
\mathcal{L}^* &= \sum_{i = 1}^N \left(\frac{y_i-\mu_{x_i}}{\sigma_{x_i}}\right)^2+2\sum_{i = 1}^N\log(\sigma_{x_i})
\end{align}
reaches its minimum. This loss is preferable over a multitude of other losses (such as KL Divergence) as it does not require defining an auxiliary ground truth probability distribution.

\subsection{Network Architecture}
We model this regressor utilizing a multi-component neural network, which must output two values (assuming we are regressing on a single target variable): the mean and standard deviation (which can also be interpreted as an uncertainty). Figure \ref{fig:network_arch} shows an overview of the generalized architecture with the various components. When applied on high-dimensional numerical input data, the feature extractor can be implemented as a deep neural network which embeds the data into a lower dimensional space; when applied on time series data, the feature extractor can be implemented as some variant of a recursive neural network (RNN); and when applied on complex natural images, a convolutional neural network can used for feature extraction. 

We typically utilize two fully connected layers as the regressor layer, where the number of nodes is determined by the complexity of the output of the feature extraction layer.  The regressor network produces the mean and the standard deviation, and \emph{Softplus} is applied to the standard deviation to ensure a valid probability distribution is generated. We demonstrate the efficacy of the architecture on different data types in Section \ref{sec:expts}.
\begin{figure}
	\centering
	\includegraphics[width = 0.5\linewidth]{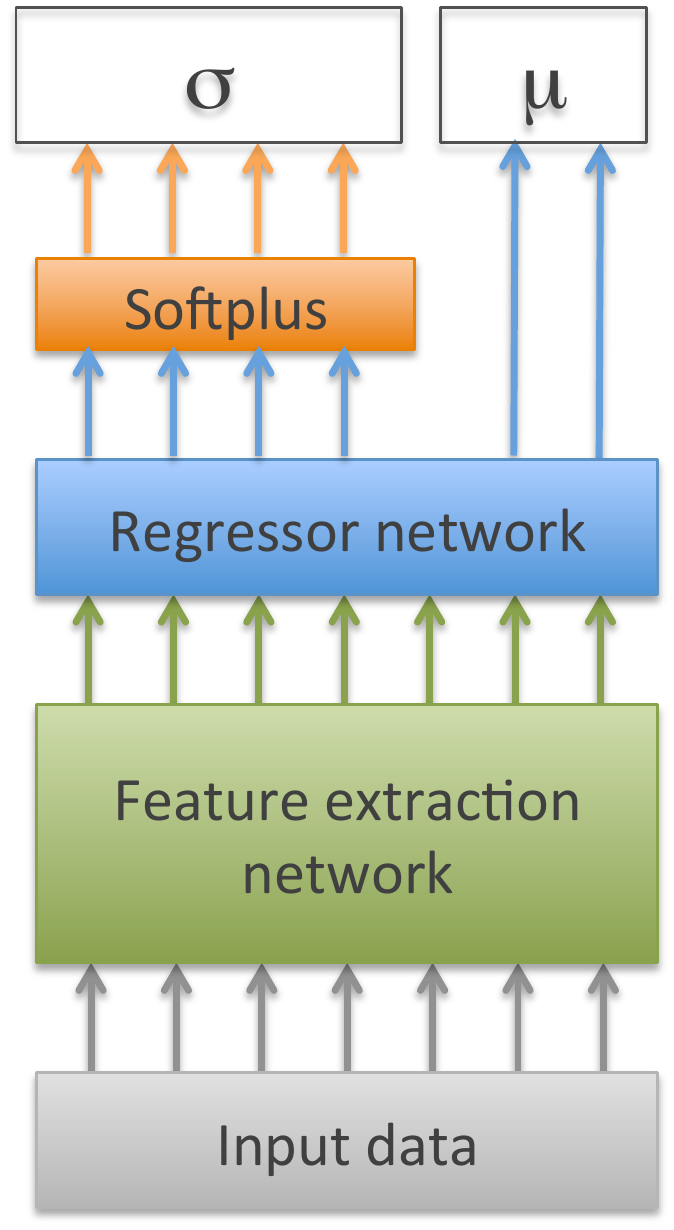}
	\caption{General architecture for the regressor.}
	\label{fig:network_arch}
\end{figure}

\section{Experiments and Results}\label{sec:expts}

To evaluate the proposed method, we utilized it in a variety of experiments described in the ensuing subsections. For the first Toy dataset, the goal was simply to test the regression capabilities of the model and visualize the uncertainty measures 2-dimensions of the well-known caloric dataset from Kaggle\footnote{\url{https://www.kaggle.com/fmendes/exercise-and-calories}}.  For the next set of experiments, we applied the method to the standard benchmark datasets commonly used to measure the quality of a regression algorithm and compare with the state-of-the-art techniques described in Section \ref{sec:priorwork}.
Next, we applied the algorithm to highly volatile, stock prices using data from 2015 to mid-year 2018  to predict the uncertainty of the stock in from mid-year 2018 till date. These uncertainty periods were shown to correspond to  different current affairs events that could have impacted that particular stock. 

Lastly, we applied the network to the complex, large-scale image dataset, the IMDb-Wiki data, to demonstrate its efficacy in performing high-quality age predictions (via regression) along with the uncertainty measures associated with the predictions. While many deep learning applications have been used successfully for different problems involving image data, there is very little work from the literature on estimating the associated uncertainties. In the different networks tested, we measure the quality of the uncertainties via the predictive log likelihood.

\subsection{Toy Caloric Datasets}\label{sec:cal}
As an initial test of our framework, we perform a one dimensional regression utilizing one parameter on a toy dataset created from Kaggle (``Exercise and Calories"). This is used purely for illustrative purposes. In addition, this will provide empirical evidence that this network is capable of capturing aleatory uncertainty (at least for basic scenarios). We attempt to determine how many calories an individual burned based on body heat. Please note that we add artificial noise to discourage the network from memorizing the mean and standard deviations of each input. 

In Figure \ref{fig:cal}, one can observe that this network successfully converges to perform the distribution regression, demonstrating that the network can capture aleatory uncertainty (the randomness inherent in the system). As this dataset is purely for demonstrative purposes, we do not include quality metrics.
\begin{figure}
    \centering
    \includegraphics[width=0.75\linewidth]{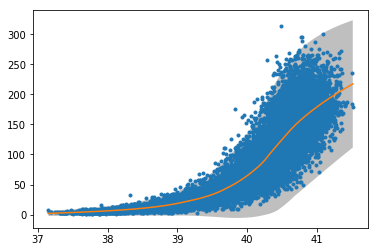}
    \caption{Regression on calories burned from body heat data. The yellow line is the regression. The gray shaded region is the $3\sigma$ confidence interval}
    \label{fig:cal}
\end{figure}

\subsection{Numerical Datasets}
\begin{table}[h]
	{\scriptsize
		\centering
		\begin{tabular}{|l|cccc|}
			\hline
			Dataset & \cite{ProbBackprop} PBP & \cite{MCDrop} MC- & \cite{Deep_ensemble} Deep  & Ours \\
			&  & Dropout & Ensembles & \\
			\hline
			Boston  & 2.57 $\pm$ 0.09 &   2.46 $\pm$ 0.25 &  2.41 $\pm$ 0.25 & \textbf{2.23 $\pm$ 0.05}\\
			Concrete & 3.16 $\pm$ 0.02  &  \textbf{3.04 $\pm$ 0.09} & \textbf{3.06 $\pm$ 0.18} & \textbf{3.05 $\pm$ 0.04}\\
			Energy & 2.04 $\pm$ 0.02 &   1.99 $\pm$ 0.09 &   \textbf{1.38 $\pm$ 0.22} & 1.91 $\pm$ 0.02\\
			Kin8nm &  -0.90  $\pm$  0.01 &  -0.95  $\pm$  0.03 &  \textbf{-1.20  $\pm$  0.02} & \textbf{-1.18 $\pm$ 0.02}\\
			Naval-  &  -3.73  $\pm$  0.01 &  -3.80  $\pm$  0.05 & \textbf{-5.63  $\pm$  0.05} & -3.82 $\pm$ 0.09\\
			propulsion &  &  &   &  \\
			Power plant & \textbf{2.84  $\pm$}  0.01 & \textbf{2.80  $\pm$  0.05 }& \textbf{2.79  $\pm$  0.04}  & \textbf{2.85 $\pm$ 0.01}\\
			Protein & 2.97  $\pm$  0.00 & 2.89  $\pm$  0.01 & 2.83  $\pm$  0.02 & \textbf{2.14 $\pm$ 0.01}\\
			Wine  & 0.97  $\pm$  0.01 & 0.93  $\pm$  0.06 & 0.94  $\pm$  0.12 & \textbf{0.87 $\pm$ 0.02}\\
			Yacht & 1.63  $\pm$ 0.02  &  1.55  $\pm$  0.12 & \textbf{1.18  $\pm$  0.21} & 4.06 $\pm$ 0.00\\
			MSD & 3.60  $\pm$  NA & v3.59  $\pm$  NA & 3.35  $\pm$  NA & 3.40 $\pm$  NA\\
			\hline		
	\end{tabular} }
	\caption{Comparison of different architectures performance for NLL on popular benchmark datasets. Measurements courtesy of Deep Ensembles paper by Lakshminarayanan et al. [7].}
	\label{tab:num}
\end{table}

 As an additional test of the framework, we demonstrate that our model is on-par or superior to other popular uncertainty quantification models (specifically PBP \cite{ProbBackprop}, MC Dropout\cite{MCDrop} and Deep Ensembles\cite{Deep_ensemble}) for regression under several benchmark datasets, commonly used to measure the quality of a regression algorithm. As can be observed in Table \ref{tab:num}, with the exception of the Yacht dataset where our technique is under-par, we performed on-par with or out-performed all other approaches  in terms of NLL (negative log-loss). 

\subsection{Uncertainty Measures on Stock Prices}
To test for uncertainty in predictions in large, complex, stochastic, highly volatile \emph{time series data}, we applied the methodology specifically on similar stocks from the entertainment industry\footnote{One of the authors spent his summer internship at a financial organization and specifically analyzed this family of stocks.}.  The family is comprised of stocks from \emph{$21^{st}$ Century Fox, Inc.~(FOX), Netflix, Inc.~(NFLX), Time Warner, Inc.~(TWX), Amazon.com, Inc.~(AMZN), Walt Disney Co.~(DIS), Comcast Corporation~(CMCSA)}. The stocks are classified as a family based on their sector, industry, asset class, and the prices of the stocks over an extended period of time being highly correlated with each other. 

Suppose we are given the stock close prices for $n$ days prior to day $T$ : $\{x_t\}_{T-n-1}^{T-1}$. We wish to predict the closing price on day $T$: $x_T$. To do this, we predict to prescribe a distribution onto $x_T \sim \mathcal{N}(\mu_T, \sigma_T)$. 
\begin{figure}[h]
	\centering
	\includegraphics[width=0.95\linewidth]{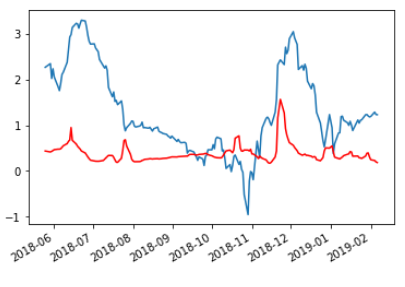}
	\caption{The blue graph is the stock price chart for {\it{FOX}} while the red graph is the measure of uncertainty estimated by the network. \scriptsize{Image is best viewed in color}}
	\label{fig:stock}
\end{figure}

\subsubsection{Data preparation and Training Schedule}
We downloaded the publicly available stock price information for the family of stocks explained above. Data for the entire family from 2015 till May 2018 was used as training data, with the goal of predicting uncertainty for only the \emph{FOX} stocks from June 2018 till date (February 2019 at the time of submission).

The network shown in Figure \ref{fig:network_arch} was implemented with the feature extraction layer being implemented as a gated recurrent unit (GRU) with a look-back of 10 days. The training scheme involved looking at the stock prices over a period of 10 days with the goal of predicting price on the 11th day along with the measure of uncertainty of the prediction. The resulting uncertainty measures are shown in Figure \ref{fig:stock}.

\subsubsection{Analysis of Results}
To analyze the uncertainties resulting from the implementation, we set a threshold of 0.5 so that days on which the uncertainty measure was above this threshold were flagged as anomalous trading days. We provide a list of \emph{FOX}-related news \footnote{News data was obtained from \url{https://www.reuters.com/finance/stocks/FOX/key-developments}. We threw away many other events leaving those related to where our uncertanties were high} in that period and compare with the anomalous days predicted by our network. The results are shown in Table \ref{tab:stocks}.
\begin{table}[h]
	{\scriptsize
		\centering
		\begin{tabular}{|c|c|l|}
			\hline
			\textbf{Real Date} & \textbf{Network predictions} & \textbf{News related to 21st Century FOX}\\
			\hline
			05-17-18 & 05-31-18 & -Suzanne Scott named CEO Of \emph{FOX} News\\
			06-13-18 & 06-15-18 & -Comcast offers to buy 21st Century Fox \\
			& & media assets for \$65B in cash\\
			10-19 till  & 10-19 till  & -Walt Disney receives unconditional approval \\
			10-20-18 & 10-22-18 & from China For 21st Century Fox deal;\\
			 & & -Amazon/Blackstone bid for Disney’s 22 \\
			 & & regional sports networks;\\
			11-26-18 & 11-26-18 & -Disney, Fox sued in U.S. for \$1B over \\
			& & Malaysia theme park \\
			01-07-19 & --- & -21st Century Fox announces filing of \\
			& & registration statement on Form 10 for Fox\\
			\hline
	\end{tabular}}
	\caption{The left column shows true dates on which major events occurred at 21st Century Fox; the second column shows the closest date estimated by our network, and the last column describes the event.}
	\label{tab:stocks}
\end{table}

\subsection{Age Estimation from Face Image}
To test how well this architecture will work on large complex datasets, we applied it on the nontrivial problem of age estimation. Given an image of a face, the network was tasked with predicting the age of the individual in the picture. Posed as a general problem, this task is a very challenging regression problem.

We utilized the IMDb-Wiki Dataset: a dataset of half a million faces scraped from both IMDb and Wikipedia (primarily IMDb), and tagged with the corresponding ages of individuals in the images. This dataset was generated by first identifying faces in images utilizing the Mathias et. al. face detector\cite{face}. The faces were then given a 40\% margin around the border and cropped out. Finally, the age was automatically extracted from the document by extracting both the time of the photograph and the year of the individual's birth.
Due to the highly automated nature of the collection of the data, this dataset is very noisy, where multiple entries in the dataset contain either no face or multiple faces. Additionally, several of the entries contain just a copyright sign. Also, in some cases, the collection year was incorrectly extracted from the webpage. See Table \ref{fig:bad} for examples of invalid face images. 

\begin{figure}
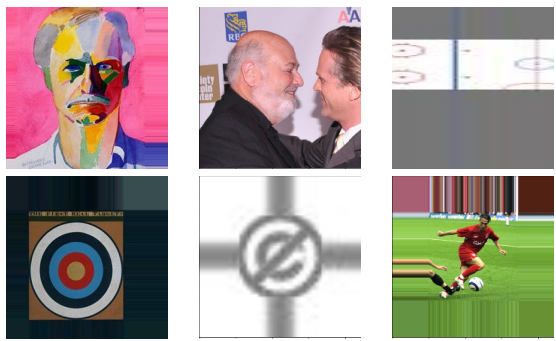

	\centering
	\begin{tabular}{ccc}
		\tablepiclarge{"image_samples/Bad/46226498_1027604394078563_7999235082512498688_n"} &
		\tablepiclarge{"image_samples/Bad/67_6145321_603558_actually_25_score_21_604"} & 
		\tablepiclarge{"image_samples/Bad/7300200_1949-06-10_1940"} \\
		\tablepiclarge{"image_samples/Bad/46403099_449840302207160_2966258797217054720_n"} &
		\tablepiclarge{"image_samples/Bad/download-15"} &
		\tablepiclarge{"image_samples/Bad/29_19400811_595427_actually_27_score_13_693"}
	\end{tabular}
	\caption{Examples of invalid data in the IMDb-Wiki dataset. Images were identified algorithmically as having extremely high uncertainty and loss values}
	\label{fig:bad}
\end{figure}

Although this dataset contained a similar distribution of males to females (see Figure \ref{fig:age_distr}(left)), it contained primarily individuals between 20 and 40 years old. Additionally, because the IMDb dataset contained a random sampling of Hollywood actors, the dataset was primarily composed of young Caucasian individuals, thus, having high implicit bias. We empirically demonstrated that our method is still capable of correctly identifying underrepresented samples in spite of the imbalances in the data.
\begin{figure}
    \centering
    \begin{tabular}{cc}
    \includegraphics[width = 0.6\linewidth]{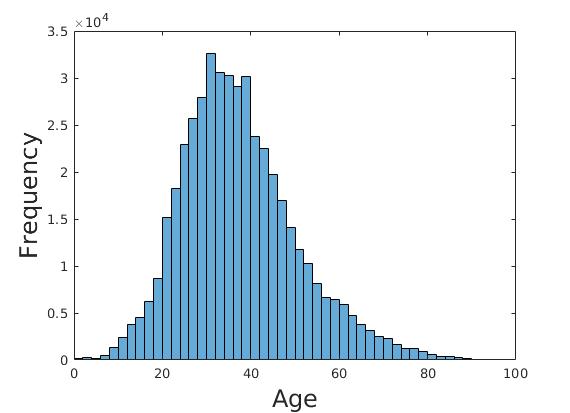} &
    \hspace*{-5mm}\raisebox{5mm}{\includegraphics[width = 0.4\linewidth]{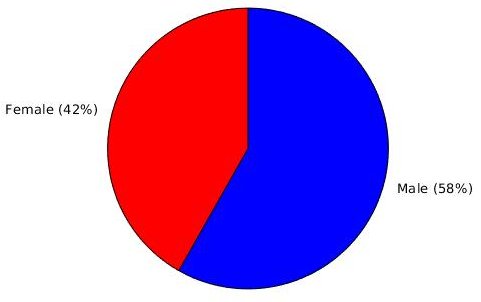}}
    \end{tabular}
    \caption{Statistics of the IMDb dataset. From the image labels provided on the left is age distribution and the right shows gender distribution.}
    \label{fig:age_distr}
\end{figure}

\subsubsection{Cleaning, Preparation and Training Schedule}
We did not wish to excessively clean the data, but rather remove the clearly wrong data. We did this by removing samples which had individuals younger than three years old or older than 100 years old. Additionally, we removed images which were too small (namely smaller than 16 by 16). We then re-sized all the images to 224 by 224.

We did \underline{not} remove invalid images which contained multiple, or no faces, although we standardize the images to be color (if they were black and white, we replicated that channel over the R, G, and B channels). Additionally, we did not remove the mislabeled entries (those which had valid ages and valid images, but were clearly mislabeled). We did not remove these so we could test the ability of the uncertainty quantification. One would expect an appropriate uncertainty quantification algorithm would give high uncertainty for invalid data or abnormal data. We exploit this later to automatically clean the dataset.

This image-based network was trained utilizing a 16-layer convolutional neural network (CNN) with an Adam optimizer, a learning rate of 0.00005 and a batch size of 8 until convergence (6 epochs).

\subsubsection{Analysis of results}
\begin{figure}[t]
	\footnotesize
	\begin{tabular}{ccccc}
		\hspace*{-3mm}
		\includegraphics[width = 0.2\linewidth]{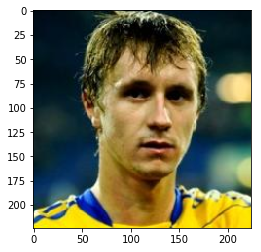}& \hspace*{-8mm}
		\includegraphics[width = 0.2\linewidth]{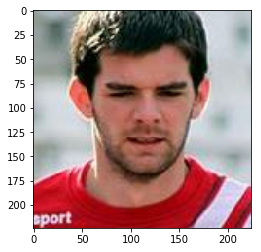} & \hspace*{-8mm}
		\includegraphics[width = 0.2\linewidth]{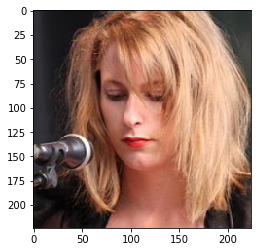} & \hspace*{-8mm}
		\includegraphics[width = 0.2\linewidth]{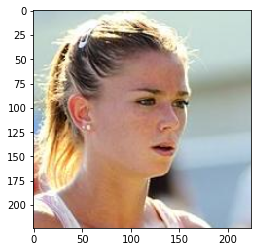} & \hspace*{-8mm}
		\includegraphics[width = 0.2\linewidth]{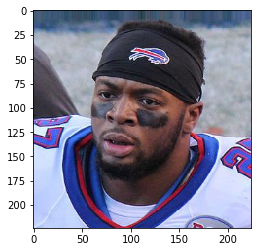}\\
		\hspace*{-3mm}
		20/19.6/3.1 & \hspace*{-3mm}
		21/22.0/3.2 & \hspace*{-3mm}
		24/22.3/7.5 & \hspace*{-3mm}
		23/23.8/7.8 & \hspace*{-4mm}
		22/22.8/3.6 \\
		\hspace*{-3mm}
		\includegraphics[width = 0.2\linewidth]{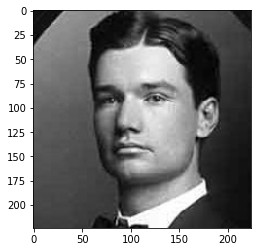} &\hspace*{-8mm}
		\includegraphics[width = 0.2\linewidth]{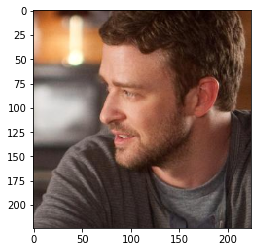} &\hspace*{-8mm}
		\includegraphics[width = 0.2\linewidth]{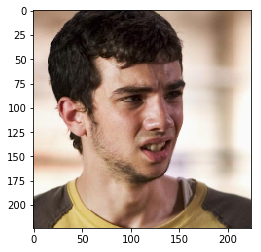} &\hspace*{-8mm}
		\includegraphics[width = 0.2\linewidth]{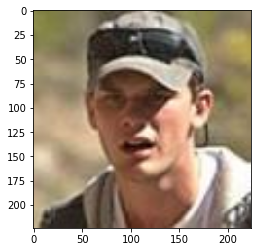} &\hspace*{-8mm}
		\includegraphics[width = 0.2\linewidth]{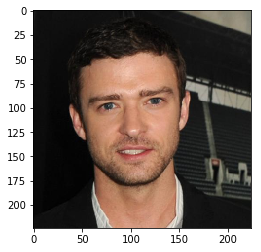}\\
		\hspace*{-3mm}
		92/26.1/7.9 & \hspace*{-3mm}
		82/30.8/6.4 & \hspace*{-3mm}
		67/27.1/5.9 & \hspace*{-3mm}
		70/25.6/6.6 & \hspace*{-3mm}
		82/35.3/7.2\\
		\hspace*{-3mm}
		\includegraphics[width = 0.2\linewidth]{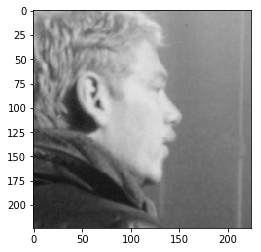} &\hspace*{-8mm}
		\includegraphics[width = 0.2\linewidth]{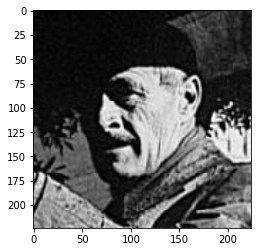} &\hspace*{-8mm}
		\includegraphics[width = 0.2\linewidth]{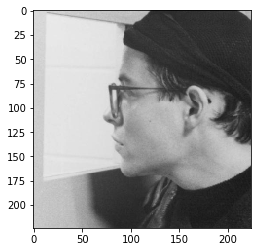} &\hspace*{-8mm}
		\includegraphics[width = 0.2\linewidth]{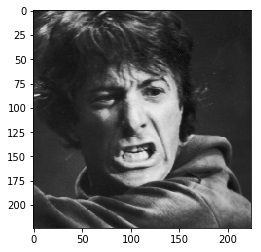} &\hspace*{-8mm}
		\includegraphics[width = 0.2\linewidth]{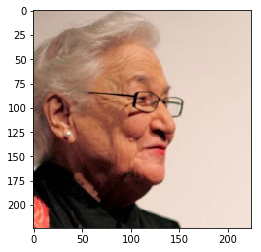}\\
		\hspace*{-3mm}
		36/55.3/19.0 & \hspace*{-3mm}
		58/62.0/19.1 & \hspace*{-3mm}
		27/69.6/21.6 & \hspace*{-3mm}
		69/76.4/20.5 & \hspace*{-3mm}
		88/83.9/20.8
	\end{tabular}
	\caption{The three numbers below each image correspond to (i) the actual age (as provided in the dataset)/(ii) the estimated age (as predicted by the regression network)/(iii) the uncertainty value reported by the network (the higher the value, the more uncertain the prediction). The top row shows some of the faces on which the network reported the lowest error values. The middle row shows the faces on which were reported the highest errors; and the last row shows the faces on which the network reported the highest uncertainty.}
	\label{tab:faceresults}
\end{figure}

\begin{table}[]
  {\footnotesize
    \centering
    \begin{tabular}{|c|c|c|}
    \hline
        \textbf{Method} & \textbf{MAE} & \textbf{NLL}\\
    \hline
        CNN + Regressor & 7.54 &  \\
    \hline
        CNN + Regressor + Uncertainty & 7.57 & 3.63 \\
    \hline
        CNN + Regressor + Uncertainty + Cleaning & 5.22 & 3.53 \\
    \hline
    \end{tabular}}
    \caption{The accuracy (both mean absolute error and negative log likelihood) of various approaches on age estimation.}
    \label{tab:face_res}
\end{table}

The results of these experiments were very promising (first row of Figure \ref{tab:faceresults}). Even on this noisy dataset, the architecture only performed poorly when the ground truth was wrong (see the middle row of Figure \ref{tab:faceresults}). These results demonstrate that our model is capable of capturing epistemic uncertainty. Additionally, this architecture's uncertainty not only expressed how confident the model was, but also how clean the data sample was. Thus, this model often reported high confidence if a sample was well represented within the dataset. Empirically we can see that difficult samples (those in a class with low representation, poor lighting, side facing faces, ambiguous individual, multiple faces) obtained high uncertainty. Images in the dataset which were incorrectly scraped along with excessively noisy or incorrect data had the highest uncertainty (see the last row of Figure \ref{tab:faceresults}). This architecture can therefore be used to evaluate the quality of samples, assuming a large portion of the data is of good quality.

\subsubsection{Determining the overhead of uncertainty quantification}
An error quantification network is often only appealing if it does not have a significant impact on performance. Thus, the quantification of the discrepancy of some error metric (say RMSE) between a classical regressor with $\omega$ parameters and that of an uncertainty-aware regressor with $\omega$ parameters should be minimized. To this end, we train two networks utilizing the same initial configuration of parameters and same number of parameters (except for the last layer) until convergence (one vanilla regressor and one error quantification regressor).

Examining Table \ref{tab:face_res}, we can observe that the discrepancy between the MAE of the uncertainty-agnostic regressor and the uncertainty-aware regressor is negligible. Thus, computing the uncertainty does not provide any significant additional overhead to this model. This final layer can therefore be added to \textit{any} regressor to provide uncertainty metrics.

\subsubsection{Automated data cleaning}
As described earlier, this architecture can be utilized to determine the quality of a sample by examining the uncertainty produced. After training, we identified the samples with the top $5\%$ uncertainty and removed them (please note that we left the validation samples unchanged). After removing these samples from our training set, we obtained significantly better results on the validation dataset (see Table \ref{tab:face_res}). Thus, this architecture is uniquely well-suited for unclean datasets to generate relatively high performing regressors.

\section{Limitations}
There are several limitations to both the uncertainty quantification method and the data cleaning process described.
\paragraph*{\sc{Uncertainty Quantification}}
While the uncertainty quantification has been demonstrated (both analytically and empirically) to appropriately quantify error, applying this direct method will likely fail if the uncertainty follows any other distribution other than normal (but this is also the case for confidence measurements in classical statistics). 

\paragraph*{\sc{Data Cleaning Process}}
While the data cleaning process was implemented successfully here and shown to benefit learning, utilizing this process can have several adverse side effects if not applied carefully. Eliminating uncertainty could increase the systematic bias of the dataset and compromise the integrity of the data. 

\section{Conclusion and Future Work}
This method of uncertainty quantification has been demonstrated to work well with large scale image datasets. Furthermore, this method has been shown to perform better than current methods for uncertainty quantification without overhead. This uncertainty quantification aspect has been exploited to develop a data cleaning procedure which improved the accuracy on an unchanged validation set. Future work for this subject includes generalizing this to arbitrary distribution regression and investigating uncertainty for classification.


\clearpage
\bibliographystyle{ieeetr}
\bibliography{egbib}
\section{Appendix}
As stated in the paper, we are learning $\rho_x(y)$ from data drawn from $p(X,Y)$

\begin{theorem}
\begin{align}
    \mathcal{L} &= -\iint\limits_S \log(\rho_{X}(Y)) p(X,Y) dS
\end{align} is minimized (under infinite i.i.d. data) when $\rho_x(y) = p(X,Y)$ if for all $X$, $p(X,Y) = \mathcal{N}(\mu_x,\sigma_x)$ for some $\mu_x, \sigma_x$.
\end{theorem}
\begin{proof}
Suppose $p(X,Y) = \mathcal{N}(\mu_x,\sigma_x)$ (i.e. for a given $x$, the $y$ is normally distributed). Then,
\begin{align}
    \mathcal{L} &= -\iint\limits_S \log(\rho_{X}(Y)) p(X,Y) dS\\
    &= -\iint\limits_S \log(\rho_{x}(y)) \frac{1}{\sqrt{2\pi\sigma_x^2}}e^{-\left(\frac{y - \mu_x}{\sigma_x}\right)^2}dy\,dx.
\end{align}
Furthermore, $\rho_X(Y) = \frac{\partial}{\partial y}F^x(y)$ for some cumulative probability distribution $F^x(y)$. Please note we notate $\frac{\partial}{\partial y}F^x(y)$ as $F^x_y(y)$ for brevity. Thus,
\begin{align}
    \mathcal{L} &= -\iint\limits_S \log(F^x_y(y))) \frac{1}{\sqrt{2\pi\sigma_x^2}}e^{-\left(\frac{y - \mu_x}{\sigma_x}\right)^2}dy\,dx.
\end{align}
By the Euler Lagrange theorem, this function reaches its minimum when
\begin{align}
    \frac{\partial \mathcal{M}}{\partial \rho} - \frac{\partial}{\partial x}\left(\frac{\partial \mathcal{M}}{\partial F^x_x}\right) - \frac{\partial}{\partial y}\left(\frac{\partial \mathcal{M}}{\partial F^x_y}\right) &= 0\\
    \frac{\partial}{\partial y}\left(\frac{\partial \mathcal{M}}{\partial F^x_y}\right) &= 0\\
    \frac{\partial^2}{\partial y\,\partial F^x_y}\left(
    \log(F^x_y(y)) \frac{1}{\sqrt{2\pi\sigma_x^2}}e^{-\left(\frac{y - \mu_x}{\sigma_x}\right)^2}
    \right)&=0
\end{align}
where $\mathcal{M}$ is the integrand of $\mathcal{L}$. Differentiating, one obtains
\begin{align}
\frac{2 \left(y-\mu _x\right) e^{-\left(\frac{y - \mu_x}{\sigma_x}\right)^2}}{\sigma_x^{2}}&=\frac{F^x_{yy}(y)
e^{-\left(\frac{y - \mu_x}{\sigma_x}\right)^2}}{F^x_y(y)}.
\end{align}
As $\rho_X(Y) = F^x_y(y)$,
\begin{align}
\frac{2 \left(y-\mu _x\right)}{\sigma_x^{2}}&=\frac{1}{\rho_x(y)} \frac{\partial\rho_x(y)}{\partial y}.
\end{align}
Thus,
\begin{align}
\int\frac{2 \left(y-\mu _x\right)}{\sigma_x^{2}}dy&=\int\frac{d\rho_x(y)}{\rho_x(y)}.
\end{align}
Solving this separable differential equation under the condition that $\int_{-\infty}^\infty\rho_x(y)dy = 1$ yields \begin{align}
F^x_y(y) =\rho_x(y) = \frac{1}{\sqrt{2\pi\sigma_x^2}}e^{-\left(\frac{y - \mu_x}{\sigma_x}\right)^2}.
\end{align}
Thus, $\rho_x(y) = p(X,Y)$ as desired.
\end{proof}
As an optimal learning scheme reaches the global optimum, under an optimal learning scheme, $\rho_x(y)$ would be learned to be $p(X,Y)$
\end{document}